\documentclass[11pt]{article}
\usepackage{amsmath,amssymb,amsfonts} 
\usepackage{epsfig} \usepackage{latexsym,nicefrac,bbm}
\usepackage{xspace}
\usepackage{color,fancybox,graphicx,subfigure,fullpage}
\usepackage[top=1in, bottom=1in, left=1in, right=1in]{geometry}
\usepackage{tabularx} \usepackage{hyperref} 
\usepackage{pdfsync}
\usepackage{multicol}
\usepackage{enumitem}
\usepackage{algorithm}
\usepackage{algpseudocode}
\usepackage{tikz}
\usetikzlibrary{calc}
\usepackage{amssymb}
\usepackage{amsmath}
\usepackage{amsthm}
\usepackage{multirow}
\usepackage{booktabs} 

\usepackage{array}
\usepackage{arydshln}
\newcommand{\cB}{\mathcal{B}}

\newcommand{\spn}{\mathrm{span}}

\newcommand{\wt}{\widetilde}

\newcommand{\cC}{\mathcal{C}}

\newcommand{\eps}{\varepsilon}

\def\qed{\hfill\ensuremath{\square}}

\newcommand\N{\mathbb N}
\newcommand\R{\mathbb R}

\newtheorem{theorem}{Theorem}[section]

\newtheorem{definition}{Definition}[section]
\newtheorem{lemma}[theorem]{Lemma}

\newcommand{\eqlabel}[1]{\label{#1}\tag{#1}}

\newcommand{\norm}[1]{\ensuremath{\left\lVert #1 \right\rVert}}
\newcommand{\inparen}[1]{\left(#1\right)}             
\newcommand{\inbraces}[1]{\left\{#1\right\}}           
\newcommand{\insquare}[1]{\left[#1\right]}             
\newcommand{\inangle}[1]{\left\langle#1\right\rangle} 
\newcommand{\deter}[1]{\operatorname{det}\left(#1\right)}

\title{\bf Fair and Diverse DPP-based Data Summarization \footnote{A short version of this paper appeared in the workshop FAT/ML 2016 - \url{https://arxiv.org/abs/1610.07183}}}

\usepackage{authblk}

\author[1]{L. Elisa Celis}
\author[2]{Vijay Keswani}
\author[3]{Damian Straszak}
\author[4]{Amit Deshpande}
\author[5]{Tarun Kathuria}
\author[6]{Nisheeth K. Vishnoi}
\affil[1,2,3,6]{\small \'{E}cole Polytechnique F\'{e}d\'{e}rale de Lausanne (EPFL), Switzerland}
\affil[4]{Microsoft Research, India}
\affil[5]{UC Berkeley}

\begin{document}

\maketitle
 
\begin{abstract}
Sampling methods  that choose a subset of the data proportional to its {\em diversity} in the feature space are popular for data summarization.
However, recent studies have noted the occurrence of {\em bias} -- under or over representation of a certain gender or race -- in  such data summarization methods.
In this paper we initiate a study of the problem of outputting a diverse and  {\em fair} summary of a given dataset. 
We work with a well-studied determinantal measure of diversity and corresponding distributions (DPPs) 
and present a framework that allows us to incorporate a  general class of fairness constraints into such distributions.
Coming up with efficient  algorithms to sample from these constrained determinantal distributions, however, suffers from a  complexity barrier and we present a fast sampler that is provably good when the input vectors satisfy a natural property.
Our experimental results on a real-world and an image dataset  show that the diversity of the samples produced by adding fairness constraints is not too far from the unconstrained case, and we also provide a theoretical explanation of it.

\end{abstract}

\newpage



\section{Introduction}
A  problem  facing many services  -- from search engines and news feeds to machine learning --  is data summarization: how can one select a small but representative, i.e., {\em diverse}, subset from a large dataset. 
For instance, Google Images outputs a small subset of images from its enormous dataset given a user query. 
Similarly, in training a learning algorithm one may be required to choose a subset of data points to train on as training on the entire dataset may be  costly. 
However, data summarization algorithms prevalent in the online world have been recently shown to be biased  with respect to sensitive attributes such as gender, race and ethnicity.
For instance, a recent study found  evidence of systematic under-representation of women in search results \cite{KayMM2015}. 
Concretely, the above work studied the output of Google Images for various search terms involving occupations and found, e.g., that for the search term ``CEO'', the percentage of women in top 100 results was $11\%$, significantly lower than the ground truth of $27\%$.
Through studies on human subjects, they also found that such  misrepresentations have the power to influence people's perception about reality. 
Beyond humans, since data summaries are used to train algorithms, there is a danger that these biases in the data might be passed on to the algorithms that use them;  a phenomena that is being revealed more and more in automated data-driven processes in education, recruitment, banking, and judiciary systems, 
see \cite{ON2016}.
A  robust and widely deployed method for data summarization is to associate a {\em diversity score} to each subset and select a subset with probability proportional to this score; see \cite{Hesabi2015}. 
This paper focuses on a concrete geometric measure of diversity of a subset $S$ of a dataset $\{v_x\}_{x \in X}$ of vectors -- the {\em determinantal measure} denoted by $G(S)$ \cite{Kulesza2012}; and the resulting probability distribution is called a determinantal point process (DPP). 
$G(S)$ generalizes the correlation measure for two vectors to multiple vectors and, intuitively, the larger $G(S)$, the more diverse is $S$ in the feature space.
Among benefits of $G(\cdot)$ are its overall simplicity, wide applicability -- not depending on combinatorial properties of the data, and efficient computability. 
A potential downside might be the additional effort required in modeling, i.e., to represent the data in a suitable vector form so that the geometry of the dataset indeed corresponds to diversity.
Despite the well-acknowledged ability of DPPs to produce diverse subsets, unfortunately, there seems to be no obvious way to ensure that this also guarantees {\em fairness} in the DPP samples in the form of appropriate representation of sensitive attributes in the subset selected.
Partially, this is due to the fact that fairness could mean different things in different contexts.
For instance, consider a dataset in which each data point has a gender.
One notion of fairness, useful in ensuring that the ground truth does not get distorted, is {\em proportional representation}: i.e., the fraction of Males (respectively Females) in the output set should be identical to that in the input dataset \cite{KayMM2015}.
Another notion of fairness, argued to be necesseary to reverse the effect of historical biases \cite{koriyama2013optimal}, could be {\em equal representation} -- the number of Males is equal to that of Females {\em independent} of the ratio in the input dataset.
While these measures of fairness  have natural generalizations to the case when the number of sensitive types is more than two, and can be refined in several ways, one thing remains common: they all operate  in the combinatorial space of sensitive attributes of the data points.
\begin{figure}
\centering
{\includegraphics[height=7cm]{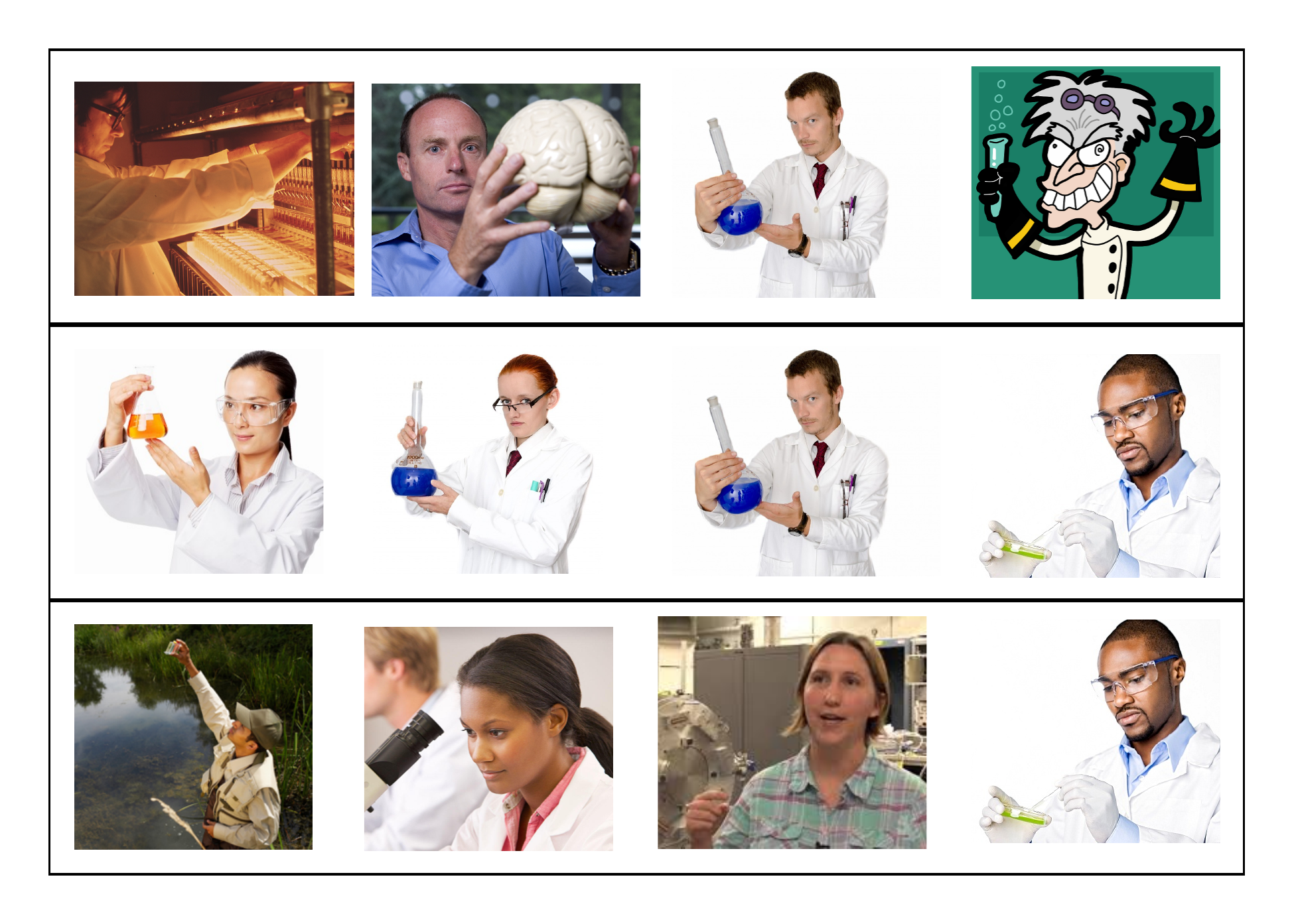}}
\caption{\small{Example sets of images displaying tradeoffs between fairness and geometric diversity. The top row of images is diverse in the geometric sense but not fair with respect to gender of race. The second row of images seems fair with respect to these sensitive features but is not diverse in the feature space. Our goal is to produce a subset of images that is visually distinct and demographically varied, as depicted in the bottom row.\vspace{-.2in}}}
\label{fig:example}
\end{figure}

Simple examples 
(see, e.g., Figure \ref{fig:example}) 
show that, in certain settings, geometric diversity does not  imply fairness and vice-versa; however, there seems to be no intrinsic barrier in attaining both. 
We initiate a rigorous study of the problem of incorporating fairness with respect to sensitive attributes of data in DPP-based sampling for data summarization.
Our contributions are: A framework  that can incorporate a wide class of notions of fairness with respect to disjoint sensitive attributes and, conditioned on being fair in the specified sense, outputs subsets where the probability of a set is still proportional to $G(\cdot)$.
In particular, we model the problem as sampling from a {\em partition} DPP -- the parts correspond to different sensitive attributes and the goal is to select a specified number of points from each.
Unfortunately, the problem of sampling from partition DPPs has been recently shown to be intractable in a  strong sense \cite{CDKSV17} and the question of designing fast algorithms for it, at the expense of being approximate,  has been open.
Our main technical result is  a linear time algorithm (see Section \ref{sec:sample_project}) to sample from partition DPPs that is guaranteed to output samples from close to the DPP distribution under a  natural condition on the data (see Definition \ref{beta-balance}).  
We  prove that random data matrices satisfy this condition  in Section \ref{sec:betaConcentration}.
Experimentally, we run our algorithm on the Adult dataset \cite{datasetUCI} and a curated image dataset with various parameter settings and observe a marked improvement in fairness without compromising geometric diversity by much.
{
A  theoretical justification of this low {\em price of fairness} is provided in Section~\ref{sec:priceOfFairness};
while there have been few works on controlling fairness, ours is the first to give a rigorous, quantitative price of
fairness guarantee in any setting.
}
Overall, our work gives a general and rigorous algorithmic solution to the problem of controlling bias in DPP-based sampling algorithms for data summarization while maximizing diversity. 
\subsection{Related Work}
DPP-based sampling has been  deployed for many data summarization   tasks including  text and  images  \cite{Kulesza2011}, videos \cite{GongCGS14}, documents \cite{Lin2012}, recommendation systems \cite{Zhou09}, and sensors \cite{Krause2008}; and the study of DPPs with additional budget or resource constraints is of importance. 
While for unconstrained DPPs there are efficient algorithms to sample \cite{HKPV05}, 
the problem of sampling from constrained DPPs is intractable; see \cite{CDKSV17},  where pseudopolynomial time algorithms for partition DPPs are presented.
There is also work on approximate MCMC algorithms for sampling from various discrete point processes (see \cite{RebeschiniK15, AGR16} and the references therein),  and algorithms that are efficient for constrained DPPs under certain restrictions on the data matrix and constraints (see \cite{LJS16} and the references therein). 
To the best of our knowledge, ours is the first algorithm for constrained DPPs that is near-linear time. 
Our algorithm is a greedy, approximate algorithm, and can be considered an extension of a similar algorithm for unconstrained DPPs given by \cite{DV06}.
Finally, our work contributes towards an ongoing effort  to measure, understand and incorporate fairness in algorithms (e.g., see \cite{BS2015,IBN2016,Dwork2012,ZafarVGG17}).

\section{Our Model} \label{sec:model}
In this section we present the formal notions,  model and other theoretical constructs studied in this paper.
$X$ will denote the dataset and we let $m$ denote its size.
We assume that for each $x \in X$, we are given a (feature) vector $v_x \in \mathbb{R}^n$, where $n \leq m$ is the dimension of the data. 
Let $V$ denote the $m \times n$ matrix whose rows correspond to the vectors $v_x$ for $x \in X$.
For a set $S \subseteq X$, we  use $V_S$ to denote the submatrix of $V$ that is obtained by picking the rows of $V$ corresponding to the elements of $S$.
We can now describe geometric diversity formally.

\begin{definition}{\bf (Geometric Diversity)}
\label{def:diversity}
Given a dataset $X$ and the corresponding feature vectors $V \in \mathbb{R}^{m \times n}$, the geometric diversity of a subset $S \subseteq X$ is defined as 
$G(S) := \deter{V_SV_S^\top},$ which is the squared volume of the parallelepiped  spanned by the rows of $V_S$.
\end{definition}

\noindent
This volume generalizes the correlation measure for two vectors to multiple vectors and, intuitively, the larger the volume, the more diverse is $S$ in the feature space; see Figure \ref{fig:volDet} for an illustration.
Geometric diversity gives rise to the following distribution on subsets known as a  determinantal point process (DPP).
\begin{definition}{\bf (DPPs and $k$-DPPs)}
\label{def:kdpp}
Given a dataset $X$ and the corresponding feature vectors $V \in \mathbb{R}^{m \times n}$, the DPP is a distribution over subsets $S \subseteq X$ such that the probability 
$\mathbb{P}[S] \propto \deter{V_SV_S^\top}.$
The induced probability distribution over $k$-sized subsets is called $k$-DPP.
\end{definition}

\noindent
 A characteristic of a DPP  measure is that the inclusion of one item makes including other similar items less likely.
Consequently, DPPs
assign greater probability to subsets of points that are diverse; for example, a DPP prefers search results that cover multiple aspects of a user's query, rather than the most popular  one.

\begin{figure*}
\centering
\begin{tikzpicture}
	\coordinate (1) at (-1,1);
	\coordinate (2) at (2,1.5);
	\coordinate (3) at (1,2);

	\draw[->,ultra thick,color=black] (0,0) -- (2);
	\draw[dashed, color=gray] (1) -- ($(2)+(1)$);
	\draw[->,ultra thick,color=black] (0,0) -- (3);
	\draw[thin, color=gray] (1) -- ($(3)+(1)$);
	
	\draw[thin, color=gray] (3) -- ($(2)+(3)$);
	\draw[thin, color=gray] ($(3)+(1)$) -- ($(2)+(3)+(1)$);
	\draw[thin, color=gray] (3) -- ($(3)+(1)$);
	\draw[thin, color=gray] ($(2)+(3)$) -- ($(2)+(3)+(1)$);
	\draw[dashed, color=gray] (2) -- ($(2)+(1)$);
	\draw[->, ultra thick,color=black] (0,0) -- (1);
	
	\draw[thin, color=gray] (2) -- ($(2)+(3)$);
	\draw[dashed, color=gray] ($(2)+(1)$) -- ($(2)+(3)+(1)$);

	\draw (0.4,1.7) node {$u_1$};
	\draw (1,0.2) node {$u_2$};
	\draw (-1.3,0.3) node {$u_3$};
    \node[above right] at (-0.1,4.5){\textbf{(A)}};
	
\end{tikzpicture}\hspace*{2pt}\begin{tikzpicture}
	\coordinate (1) at (-1,1);
	\draw[->,ultra thick, color=black] (0,0) -- (3,0);
	\draw[dashed,color=gray] (1) -- ($(3,0)+(1)$);
	\draw[->,ultra thick,color=black] (0,0) -- (0,3);
	\draw[ultra thin,color=gray] (1) -- ($(0,3)+(1)$);
	
	\draw[thin, color=gray] (0,3) -- (3,3);
	\draw[thin, color=gray] ($(0,3)+(1)$) -- ($(3,3)+(1)$);
	\draw[thin, color=gray] (0,3) -- ($(0,3)+(1)$);
	\draw[thin, color=gray] (3,3) -- ($(3,3)+(1)$);
	\draw[dashed, color=gray] (3,0) -- ($(3,0)+(1)$);
	\draw[->, ultra thick,color=black] (0,0) -- (1);
	
	\draw[thin, color=gray] (3,0) -- (3,3);
	\draw[dashed, color=gray] ($(3,0)+(1)$) -- ($(3,3)+(1)$);

	\draw (-0.3,1.7) node {$u_1$};
	\draw (2,0.2) node {$u_2$};
	\draw (-0.9,0.3) node {$u_3$};
	\node[above right] at (-0.1,4.5){\textbf{(B)}};
\end{tikzpicture}\hspace*{3pt}\begin{tikzpicture}
	\coordinate (1) at (2.5,0.2);
	\coordinate (2) at (3,0);
	\coordinate (3) at (-2,0.2);
	
	\draw[->,ultra thick,color=black] (0,0) -- (2);
	\draw[thin, color=gray] (1) -- ($(2)+(1)$);
	\draw[->,ultra thick,color=black] (0,0) -- (3);
	\draw[thin, color=gray] (1) -- ($(3)+(1)$);
	
	\draw[dashed, color=gray] (3) -- ($(2)+(3)$);
	\draw[thin, color=gray] ($(3)+(1)$) -- ($(2)+(3)+(1)$);
	\draw[thin, color=gray] (3) -- ($(3)+(1)$);
	\draw[dashed, color=gray] ($(2)+(3)$) -- ($(2)+(3)+(1)$);
	\draw[thin, color=gray] (2) -- ($(2)+(1)$);
	\draw[->, ultra thick,color=black] (0,0) -- (1);
	
	\draw[dashed, color=gray] (2) -- ($(2)+(3)$);
	\draw[thin, color=gray] ($(2)+(1)$) -- ($(2)+(3)+(1)$);

	\draw (0.2,0.4) node {$u_1$};
	\draw (1,-0.3) node {$u_2$};
	\draw (-1,-0.3) node {$u_3$};
	\node[above right] at (-0.1,4){\textbf{(C)}};
\end{tikzpicture}
\caption{\textmd{\small{ \textbf{(A)} depicts how diversity relates to the volume of the parallelepiped formed by the feature vectors: more the volume,  more the diversity.
All the vectors in \textbf{(B)} are pairwise orthogonal and their collection has a large determinant and, hence, the parallelepiped has a large volume.
The parallelepiped in \textbf{(C)}, has a low volume which tends to zero as the angle between $u_1, u_2$ decreases or between $u_2,u_3$ increases. For a matrix with these vectors as rows, the determinant will be small, since the orthogonal projection of $u_1$ on $u_2$ is very small, and similarly for $u_2,u_3$. If they become parallel, the determinant becomes zero since one row is then linearly dependent on another.\vspace{-.2in}}}}
\label{fig:volDet}
\end{figure*}

\subsection{Our Algorithmic Framework}
We are given a dataset $X$ along with corresponding feature vectors $V \in \mathbb{R}^{m \times n}$ and a positive number $k \leq m$ that denotes the size of the subset or summary that needs to be generated.
The dataset $X$ is partitioned into $p$ disjoint classes $X_1 \cup X_2\cup \cdots \cup X_p$, each corresponding to a sensitive class.
A key feature of our model is that we do not fix one notion of fairness; rather, we allow for the {\em specification} of  {\em fairness constraints} with respect to these sensitive classes.
This is to make the model flexible and widely applicable in the light of the observation that, in different contexts, fairness could mean very different things.
Formally, we do this by taking as input $p$ natural numbers $(k_1, k_2, \ldots, k_p)$ such that $\sum_{j=1}^p k_j=k$ is the sample size.
These numbers give rise to a {\em fair family of allowed subsets} defined to be
$\cB:=\{S\subseteq X: |S\cap X_j|=k_j \mbox{ for all }j=1,2, \ldots, p\}.$
The generality of our framework is evident:  by setting $(k_1,\ldots,k_p)$ appropriately, the user may ensure their desired notion of fairness depending on the context. 
To give some examples, if in the  dataset the number of the $i$-th sensitive attribute is $m_i$, then we can set $k_i:={km_i}/{m} $ to obtain proportional representation.
Similarly, equal representation can be implemented by setting $k_i={k}/{p}$ for all $i$.

The fair data summarization problem then becomes to sample from a  distribution that is supported on  $\cB$. 
However, there could be many distributions supported on $\cB$  and we pick  one that is ``closest'' to the to the $k$-DPP described by $V$. 
We use the Kullback-Leibler (KL) divergence between  distributions $q$ and $\tilde{q}$ defined as  
$D_{KL}(q||\tilde{q}) := \sum_{S} q_S \log \frac{q_S}{\tilde{q}_S}.$
The following lemma characterizes  the distribution supported on $\cB$ that has the least KL-divergence to a given distribution. The proof appears in Section \ref{proof-opt-kl}.
\begin{lemma}\label{lem:opt-kl-dist}
Given a distribution $\tilde{q}$ with support set $\mathcal{C}$, let $\mathcal{B} \subseteq \mathcal{C}$ and $q$ be any distribution on $\mathcal{B}$. Then the optimal value of 
$\min_q D_{KL}(q||\tilde{q})$
is achieved by the distribution $q^\star$, such that $q^\star_S \propto \tilde{q}_S$, for $S \in \mathcal{B}$ and $0$ otherwise.
\end{lemma}
\noindent
Thus, the distribution above can be thought of as the most diverse while being fair;  we call it partition DPP, or $P$-DPP.
\begin{definition}{\bf ($P$-DPP)}
\label{def:pdpp}
Given a dataset $X$, the corresponding feature vectors $V \in \mathbb{R}^{m \times n}$,  a partition $X = X_{1} \cup X_{2} \cup \cdots \cup X_{p}$ into $p$ parts, and natural numbers $k_1,\ldots,k_p$,  $P$-DPP defines a distribution $q^\star$ over subsets $S \subseteq X$ of size $k = \sum_{i=1}^p k_i$ such that for all $S\in \cB$ we have
$q^\star_S:=\frac{\det(V_SV_S^\top)}{\sum_{T\in \cB} \det(V_TV_T^\top)},$
 and $q^\star_S = 0$ otherwise.
\end{definition}
\noindent
From the algorithmic perspective, the main problem we study is that of coming up with efficient algorithms to sample from $P$-DPPs.
The flexibility that our framework provides in specifying the fairness constraints comes at a computational cost -- coming up with algorithms to sample from $P$-DPPs. 
This is a significant challenge, especially given the results of \cite{CDKSV17} that show that sampling from $P$-DPPs is $\#$P-hard.

\section{Our Algorithm}

\subsection{Notions of Volume and Projection.}
Let us recall the interpretation of determinants in terms of volumes.
For $S \subseteq X$, $V_S$ is the set of vectors $\{v_x\}_{x \in S}$. 
If the vectors in $S$ are pairwise orthogonal, then the matrix $V_S V_S^\top$ is diagonal with entries $\{\norm{v_x}^2\}_{x \in S}$ on the diagonal and, hence, $\det(V_S V_S^\top)=\prod_{x\in S} \norm{v_x}^2$. 
In the general case, the determinant is not simply the (squared) product of the norms of vectors, however a similar formula still holds.
Let $H\subseteq \R^n$ be any linear subspace and $H^{\bot}$ be its orthogonal complement, i.e., 
$H^\bot := \{y \in \R^n \mid \inangle{x,y} = 0 \text{ for all } x \in H\}.$
Let $\Pi_H: \R^n \to \R^n$ be the orthogonal projection operator on the subspace $H^{\bot}$, i.e., whenever $w\in \R^n$ decomposes as $w_1+w_2$ for $w_1\in H$ and $w_2\in H^{\bot}$, then $\Pi_H(w)=w_2$. 
By a slight abuse of notation, we also denote by $\Pi_v$ the operator that projects a vector to another that is orthogonal to a given  vector $v\in \R^n$, i.e.,
$\Pi_v(w):=w - \inangle{w,v}/\norm{v}^2.$

The following lemma is a simple generalization of the formula derived above for orthogonal families of vectors and inspires our algorithm for $P$-DPPs. The proof of this lemma is presented in Section~\ref{proof:det_product}.
\begin{lemma}[Determinant Volume Lemma]
\label{lemma:det_product}
Let $w_1, \ldots, w_k \in \R^n$ be the rows of a matrix $W\in \R^{k \times n}$, then 
$\det(WW^\top)=\prod_{i=1}^k \norm{\Pi_{H_i}w_i}^2,$
where $H_i$ is the subspace spanned by  $\{w_1, \ldots, w_{i-1}\}$ for all $i=1,2, \ldots, k$.
\end{lemma}

\subsection{Our Sample and Project Algorithm}\label{sec:sample_project}

Before we describe our algorithms for sampling from $P$-DPPs, it is instructive to consider the special case of $k$-DPPs itself and the simple ``orthogonal'' scenario -- where all the vectors $v_x$, for $x\in X$, are pairwise orthogonal. 
In such a case, there is a simple iterative algorithm: sample $x\in X$ with probability $\propto \norm{v_x}^2$, then add $x$ to $S$ and remove $x$ from $X$; repeat until $|S|=k$. 
It is intuitively clear, and not hard to prove, that the final probability of obtaining a given set $S$ as a sample is  proportional to $\prod_{x\in S} \norm{v_x}^2=\det(V_S V_S^\top)$ and, hence, recovers the $k$-DPP exactly.

In case of $P$-DPPs where all the vectors are pairwise orthogonal, and we need to sample $k_i$ vectors from partition $X_i$, we can sample the required number of elements from each partition independently using the procedure in the previous paragraph.
The orthogonality of the vectors and the disjointness of the parts implies that this sampling procedure gives the right probability distribution.

However, when the vectors $v_x$ are no longer pairwise orthogonal,  the above heuristic can fail miserably. 
This is where we invoke Lemma \ref{lemma:det_product}.
It suggests the following strategy: once we select a vector, then we should orthogonalize all the remaining vectors with respect to it before repeating the sampling procedure.
For the case of $k$-DPPs, it can be shown that this heuristic outputs a set $S$ with probability no more than $k!$ times its desired probability  \cite{DV06}.
The $k!$ term  is primarily because the $k$ vectors can be chosen in any of the $k!$ orders.
Taking this simple heuristic as a starting point and incorporating an additional idea to deal with partition constraints,  we arrive at our {\it Sample and Project} algorithm -- see Algorithm~\ref{alg:iter_pdpp}.

\begin{algorithm}
\caption{Approximate sampling algorithm for $P$-DPPs}\label{alg:iter_pdpp}
\begin{algorithmic}[1]
\Procedure{Sample-And-Project}{$V, (X_1,.., X_p), (k_1,.., k_p)$}
\State $S \gets \emptyset$
\State $k \gets k_1 + k_2 + \dots + k_p$
\State Let $w_x := v_x$ for all $x \in X$
\While{$|S| < k$}
\State Pick any $i \in \{1, \dots, p\}$ such that $|S \cap X_i| < k_i$
\State Define $q \in \mathbb{R}^{X_i}$ by $q_x := \norm{w_x}^2$ for $x \in X_i$
\State Sample $\tilde{x} \in X_i$ from distribution $\inbraces{\frac{q_x}{\sum_{y\in X_i}q_y}}_{x\in X_i}$
\State $S\gets S \cup \{\tilde{x}\}$
\State Let $v := w_{\tilde{x}}$
\State For all $x \in X$, set $w_x := \Pi_v (w_x)$
\EndWhile
\State \textbf{return} $S$
\EndProcedure
\end{algorithmic}
\end{algorithm}

Given that we have made several simplifications and informal ``jumps'' when deriving the algorithm one cannot expect that the distribution over sets $S$ produced by Algorithm~\ref{alg:iter_pdpp} to be exactly the same as $P$-DPP.
Later in this section we give evidence that in fact the distribution output by the ``Sample and Project'' heuristic can be formally related to the $P$-DPP distribution, and hence the constructed algorithm is provably an approximation to a $P$-DPP.
However, we first note an attractive feature of this algorithm -- it is fast and practical. For a $V \in \R^{m \times n}$ matrix and $k = \sum_{i=1}^p k_i$, Algorithm~\ref{alg:iter_pdpp} can be implemented in $O(mnk)$ time.

\noindent
Note that the size of the data for this problem is already $\Theta(mn)$, hence, the algorithm does only linear work per sampled point.
For $P$-DPPs there is only one known exact algorithm which samples in time $m^{O(p)}$, which is polynomial only when $p = O(1)$ \cite{CDKSV17}. 

Another possible approach for sampling from DPPs is the Markov Chain Monte Carlo method. 
It was proved in~\cite{AGR16} that Markov Chains can be used to sample from $k$-DPPs in time roughly $\wt{O}(mk^4 + mn^2)$ given a ``warm start'', i.e., a set $S_0$ of significant probability. 
This approach does not extend to $P$-DPPs -- indeed in~\cite{AGR16} the underlying probability distribution is required to be Strongly Rayleigh, a property which holds for $k$-DPPs, but fails for $P$-DPPs whenever the number of parts is at least two.
One can still formulate an analogous MCMC algorithm for the case of $P$-DPPs -- it fails on specially crafted ``bad instances'' but seems to perform well on real world data. However, even ignoring the lack of provable guarantees for this algorithm, it does not seem possible to reduce its running time below $O(mk^4+mn^2)$, which significantly limits its practical applicability.

\subsection{Provable Guarantees for Our Algorithm}

We now present a theorem which connects the output distribution of Algorithm~\ref{alg:iter_pdpp} to the corresponding $P$-DPP.
To establish such a guarantee we require the following assumption on the singular values of the matrices $V_{X_i}$. 

\begin{definition}[$\beta$-balance] \label{beta-balance}
Let $X$ be a set of $m$ elements partitioned into $p$ parts $X_1, \dots, X_p$ and let $V \in \mathbb{R}^{m \times n}$ be a matrix. Denote by $\sigma_1 \geq \cdots \geq \sigma_n$ the singular values of $V$ and for each $i\in \{1,2, \ldots, p\}$, let $\sigma_{i,1} \geq \cdots \geq \sigma_{i,n}$ denote the singular values of $V_{X_i}$. For $\beta \geq 1$, the partition $X_1, \ldots, X_p$ is called {$\beta$-balanced} with respect to $V$ if for all $i \in \{1, \dots, p\}$ and for all $j \in \{1, \dots, n\}$,
$\sigma_{i,j} \geq \frac{1}{\beta} \sigma_j.$

\end{definition}

\noindent 
The $\beta$-balance property  informally requires that the diversity within each of the partitions $V_{X_i}$, relative to $V$, is significant. 
 A more concrete geometric way to think about this condition is as follows: if one thinks of the positive semidefinite matrix $V^\top V\in \R^{n \times n}$ as representing an ellipsoid in $\R^n$ whose axes are the singular values, then the $\beta$-balance condition essentially says that the ellipsoids corresponding to each of the partitions are a $\beta$-approximation to that of $V$.

One can construct simple examples that motivate the necessity of such a condition.\footnote{Consider an example with $p=2$ parts and $m=3n$ vectors of dimension $2n$, where the first part contains vectors $e_1, e_2, \ldots, e_{2n}$ (where $e_i$ denotes the $i$th standard basis vector) and the second part consists of $e_1, e_2, \ldots, e_n$. Such a partition is not $\beta$-balanced for any $\beta>0$ since $V$ has $2n$ non-zero singular values and $V_{X_2}$ has only $n$ of them ($V_{X_1}$ has $2n$ of them). The Sample and Project algorithm indeed fails to approximate the $P$-DPP, as it outputs a set with non-zero determinant with exponentially small probability.}
For a positive and negative example of $\beta$-balanced property, see Figure \ref{fig:betaFigure}.

\begin{figure}[t]
\centering
\begin{tikzpicture}
	\coordinate (1) at (-1,2);
	\coordinate (2) at (2,2);
	\coordinate (3) at (2,0);

	\draw[->,color=white] (0,-1.2) -- (2,-1.2);

	\draw[->,ultra thick,color=red] (0,0) -- (2);
	\draw[->,ultra thick,color=blue] (0,0) -- (3);
	\draw[->, ultra thick,color=blue] (0,0) -- (1);
	
	\draw (1,1.7) node {$w_1$};
	\draw (2,0.3) node {$w_2$};
	\draw (-0.3,1.5) node {$w_3$};
	
    \draw[->,thin, color=gray] (0,0) -- (2.5,0);
	\draw (2.5,-0.2) node {$x$};
	\draw[->,thin, color=gray] (0,0) -- (0,3);
	\draw (0.3,3) node {$y$};
    \node[above right] at (1,3){\textbf{(A)}};

\end{tikzpicture}
\hspace{2cm}
\begin{tikzpicture}
	\coordinate (1) at (-1.5,1.5);
	\coordinate (2) at (2,2);
	\coordinate (3) at (1,-1);
	\coordinate (4) at (-1,2);
	\coordinate (5) at (2,0);

	\draw[->,ultra thick,color=red] (0,0) -- (2);
	\draw[->,color=blue] (0,0) -- (3);
	\draw[->, color=blue] (0,0) -- (1);
	\draw[->,  ultra thick,color=blue] (0,0) -- (4);
	\draw[->, ultra thick,color=blue] (0,0) -- (5);
	\draw[->,  densely dotted, color=blue] (4) -- (1);
	\draw[->, densely dotted, color=blue] (5) -- (3);

	\draw (1,1.7) node {$w_1$};
	\draw (2,0.3) node {$w_2$};
	\draw (-0.3,1.5) node {$w_3$};
	\draw (-1.25,0.4) node {$\Pi_{w_1}(w_3)$};
	\draw (0,-0.85) node {$\Pi_{w_1}(w_2)$};
	
    \draw[->,thin, color=gray] (0,0) -- (2.5,0);
	\draw (2.5,-0.2) node {$x$};
	\draw[->,thin, color=gray] (0,0) -- (0,3);
	\draw (0.3,3) node {$y$};
    \node[above right] at (1,3){\textbf{(B)}};

\end{tikzpicture}
\caption{\textmd{\small{This figure represents an iteration of the algorithm for input $X = \{1,2,3\}$, $V_{X_1} = \{w_1\}$ (red) and $V_{X_2} = \{w_2,w_3\}$ (blue). If the algorithm selects the partition $X_1$ and samples the vector $w_1$, it removes the projection of $w_1$ from $w_2$ and $w_3$ to obtain $\Pi_{w_1}(w_2)$ and $ \Pi_{w_1}(w_3)$.
 }}}
\label{fig:algoEx}
\end{figure}

Importantly, Algorithm~\ref{alg:iter_pdpp} never outputs a set $S\notin \cB$, hence the only way its output distribution could significantly differ from the $P$-DPP would be if certain sets $S\in \cB$ appeared  in the output with larger probabilities than specified by the $P$-DPP. 
Our main theoretical result for {\it Sample and Project} is that for $\beta$-balanced instances we can control the scale at which such a violation can happen.
\begin{theorem}[Approximation Guarantee]\label{thm:approx_guarantee}
Let $X$ be a set of $m$ elements partitioned into $p$ parts $X_1, \dots, X_p$, a matrix $V \in \mathbb{R}^{m \times n}$ and integers $k_1, \dots, k_p$, such that $X_1, \dots, X_p$ is a $\beta$-balanced  partition with respect to $V$ and $\sum_{j=1}^p k_j$. Let $\mathcal{B}\subseteq 2^X$ denote the following family of sets
\[\mathcal{B} := \{S \subseteq X : |S \cap X_j| = k_j \mbox{ for all }j=1,2, \ldots, p\}\]
Then Algorithm 1, with $V$, $(X_1, \dots, X_p)$ and $(k_1, \dots, k_p)$ as input, returns a subset $S \in \mathcal{B}$ with probability
$\tilde{q}(S) \leq \eta_k \cdot \beta^{2k} \cdot q^\star_S$
where $q^\star_S = \frac{\det(V_SV_S^\top)}{\sum_{T \in \mathcal{B}} \det(V_{T}V_{T}^\top)}$, $k = \sum_{j=1}^p k_j$ and $\eta_k = k_1! \cdot k_2! \cdots k_p!$.
\end{theorem}

\noindent 
The proof of the approximation guarantee uses techniques inspired by \cite{DV06} who prove a similar bound for $k$-DPP sampling.

We  use the following lemmas in the proof of the theorem. The proof of these lemmas appear in Section~\ref{proof:singularValues} and Section~\ref{proof:nondegeneracy}.
\begin{lemma}\label{lemma:singularValues}
For any matrix $V \in \mathbb{R}^{m \times n}$ with $m\geq n \geq k$, 
\[\sum_{i_1 < i_2< \cdots < i_{k}} \sigma_{i_1}^2 \sigma_{i_2}^2 \cdots \sigma_{i_{k}}^2 = \sum\limits_{S:|S| = k} \det(V_SV_S^\top)\]
where $\sigma_1, \sigma_2, \dots, \sigma_n$ are the singular values of $V$ and $V_S$ is the sub-matrix of $V$ with rows  corresponding to $S$.
\end{lemma}
\noindent 
\begin{lemma} \label{lemma:nondegeneracy}
Given a $\beta$-balanced partition, Algorithm~\ref{alg:iter_pdpp} returns a set $S$ such that $\det(V_SV_S^\top)$ is non-zero with probability one.
\end{lemma}
\noindent
We use also the following low rank approximation lemma in the proof of Theorem~\ref{thm:approx_guarantee}. 
\begin{lemma}[Low Rank Approximation, see e.g. \cite{golub2012matrix}]\label{fact:lowRankApprox}
For a matrix $A \in \R^{m \times n}$, with $m\geq n$, let $A = \sum_{j=1}^m \sigma_j u_j z_j^\top$ be its singular value decomposition. Then  
$A' = \sum_{j=1}^k \sigma_ju_jz_j^\top$
is the best rank $k$ approximation of $A$, i.e., 
$$\min\limits_{\substack{B : \text{ rank}(B)=k}}\norm{A-B}^2_F$$ is achieved for $B = A'$ and attains the value $\sum_{j=k+1}^n \sigma_j^2$.
\end{lemma}

\begin{figure}
\centering
\begin{tikzpicture}
  \draw[thin,white!100] (0,0) grid (3,3);
  \draw[<->] (0,0)--(3,0) node[right]{$x$};
  \draw[<->] (0,0)--(0,3) node(yaxis)[above]{$y$};
  \draw[line width=1pt,blue,-stealth](0,0)--(2,0) node[anchor=north east]{$\boldsymbol{v_1}$};
  \draw[line width=1pt,blue,-stealth](0,0)--(2,0.3) node[anchor=south east]{$\boldsymbol{v_2}$};
  \draw[line width=1pt,red,-stealth](0,0)--(0,2) node[anchor=north east]{$\boldsymbol{v_3}$};
  \draw[line width=1pt,red,-stealth](0,0)--(0.3,2) node[anchor=south west]{$\boldsymbol{v_4}$};
  \node[above right] at (yaxis.south east){\textbf{(A)}};
\end{tikzpicture}
\hspace{2cm}
\begin{tikzpicture}
  \draw[thin,white!100] (0,0) grid (3,3);
  \draw[<->] (0,0)--(3,0) node[right]{$x$};
  \draw[<->] (0,0)--(0,3) node(yaxis)[above]{$y$};
  \draw[line width=1pt,blue,-stealth](0,0)--(2,0) node[anchor=north east]{$\boldsymbol{v_1}$};
  \draw[line width=1pt,blue,-stealth](0,0)--(2,3) node[anchor=south east]{$\boldsymbol{v_2}$};
  \draw[line width=1pt,red,-stealth](0,0)--(0,2) node[anchor=north east]{$\boldsymbol{v_3}$};
  \draw[line width=1pt,red,-stealth](0,0)--(3,2) node[anchor=south west]{$\boldsymbol{v_4}$};
    \node[above right] at (yaxis.south east){\textbf{(B)}};
\end{tikzpicture}
\caption{\textmd{\small{Suppose matrix $V$ has vectors $v_1, v_2, v_3, v_4$ as rows, and partitions $V_{X_1}$ contains $v_1,v_2$ and  $V_{X_2}$ contains $v_3, v_4$. 
\textbf{Negative Example (A)} :  For $v_1 = (2,0), v_2 = (2, \varepsilon),$ $v_3 = (0,2), v_4 = (\varepsilon, 2)$, as $\varepsilon$ goes to zero, both non-zero singular values of $V$ approach $2\sqrt{2}$. However for both $V_{X_1}$ and $V_{X_2}$, the smallest singular value approaches $0$ as $\varepsilon$ decreases. 
\textbf{Positive Example (B)} :  For $v_1 = (2,0), v_2 = (2, 3),$ $v_3 = (0,2), v_4 = (3, 2)$, the singular values of $V$ are $5.38$ and $2.23$. The singular values of both $V_{X_1}$ and $V_{X_2}$ are $3.81$ and $1.57$, which is more than half of the corresponding singular values of $V$. Therefore $X_1, X_2$ is $\beta$-balanced for $\beta=2$.}}
\vspace{-.2in}} \label{fig:betaFigure}
\end{figure}

\begin{proof}[Proof of Theorem~\ref{thm:approx_guarantee}]
Let $\pi$ be the random variable representing the ordered output of the algorithm. Suppose that the algorithm outputs the set $S = \{x_1, \dots, x_k\}$. Since the partition $X_1, \dots, X_p$ is $\beta$-balanced with respect to $V$, by Lemma \ref{lemma:nondegeneracy} the algorithm will always output a set which has non-zero determinant value, i.e, $\det(V_SV_S^\top)\neq 0$.
Consider any ordering of the set $S$, say, $\tau := (x_1, \dots, x_k)$. Let $H_j\subseteq \R^n$ denote the linear subspace spanned by the vectors corresponding to the first $j-1$ elements, i.e., $\{v_{x_1}, \dots, v_{x_{j-1}}\}$. We also define a mapping $f : X \rightarrow \{1, \dots, p\}$ such that $f(x) = i$ if $x \in X_i$.

In the first iteration say we choose partition $X_1$. Then the algorithm will sample an element from $X_1$ with probability proportional to the squared norm of the vector. After $(j-1)$ iterations $w_x$ will be the orthogonal projection of $v_x$ onto the subspace orthogonal to $\spn\{v_{x_1}, v_{x_2}, \dots, v_{x_{j-1}}\}$. This is a consequence of the fact that
$$(\Pi_{v_{x_1}}\Pi_{v_{x_1}}\cdots \Pi_{v_{x_{j-1}}})=\Pi_{H_{j}}.$$

\noindent 
Hence in the $(j-1)$-th iteration, $w_x = \Pi_{H_{j}} (v_x)$ for all $x \in X$. Therefore, the probability that the sequence $\tau$ is the output of the algorithm is
\begin{align*} 
&\mathbb{P}(\pi = \tau)
= \prod_{j=1}^k \frac{\norm{\Pi_{H_{j}}(v_{x_j})}^2}{\sum\limits_{x \in X_{f(x_j)}} \norm{\Pi_{H_{j}}(v_{x})}^2.} \eqlabel{1} 
\end{align*}
The numerator of above is $\det(V_SV_S^\top)$ by Lemma~\ref{lemma:det_product}.
Let $D_{x_1, \dots, x_k}$ denote the denominator. For each term in the denominator
$\sum\limits_{x \in X_l} \norm{\Pi_{H_j}(v_x)}^2 = \norm{V_{X_l} - V_{X_l}'}_F^2
$
where $\norm{\cdot}_F$ denotes the Frobenius norm and $V_{X_l}'$ is the rank $j-1$ matrix with rows $\{v_{x}'\}_{x \in X_l}$ such that $v_x'$ is the projection of vector $v_x$ on  $H_j$. By a  result on low rank approximations (see Lemma~\ref{fact:lowRankApprox}), we can bound the above quantity as 
\begin{align*}
\sum\limits_{x \in X_l} \norm{\Pi_{H_j} (v_x)}^2  &\geq  \sum\limits_{t=j}^n \sigma_{l,t}^2 \geq \frac{1}{\beta^2}\sum\limits_{t=j}^n \sigma_t^2
\end{align*}
where $\sigma_{l,t}$ is the $t$-th singular value of $V_{X_l}$ and second inequality is due to the $\beta$-balanced property of the partition.
Using above, the denominator of \eqref{1} becomes
\begin{align*} 
D_{x_1, \dots, x_k} &\geq \prod_{j=1}^k\frac{1}{\beta^2} \sum\limits_{t=j}^n \sigma_t^2 \geq \frac{1}{\beta^{2k}} \sum\limits_{ t_1 < \cdots < t_{k}} \sigma_{t_1}^2 \cdots \sigma_{t_{k}}^2.
\end{align*}
\noindent 
By applying Lemma~\ref{lemma:singularValues}, it then follows
$$
D_{x_1, \dots, x_k}  \geq \frac{1}{\beta^{2k}} \sum\limits_{|S| = k} \det(V_SV_S^\top) \geq \frac{1}{\beta^{2k}} \sum\limits_{S \in \mathcal{B}} \det(V_SV_S^\top).$$
Thus,  $\mathbb{P}(\pi = \tau) \leq \beta^{2k} \frac{\det(V_SV_S^\top)}{\sum_{T \in \mathcal{B}} \det(V_TV_T^\top) }.$
Since the order in which the partitions are considered by the algorithm is fixed, the vectors of each $X_i$ in $\tau$ can be permuted amongst themselves and the output set will still be $S$. Correspondingly there are $\eta_k = k_1! \cdot k_2! \cdot \cdot \cdot k_p!$ valid permutations of $\tau$. Let $T_S$ be the set of all valid permutations of elements of $S$, then
$
\tilde{q}_S = \sum\limits_{\tau \in T_S} \mathbb{P}(\pi = \tau) \leq \eta_k \cdot \beta^{2k} \cdot q^\star_S.
$\end{proof}

\begin{figure*}[t]
\includegraphics[height=4.2cm]{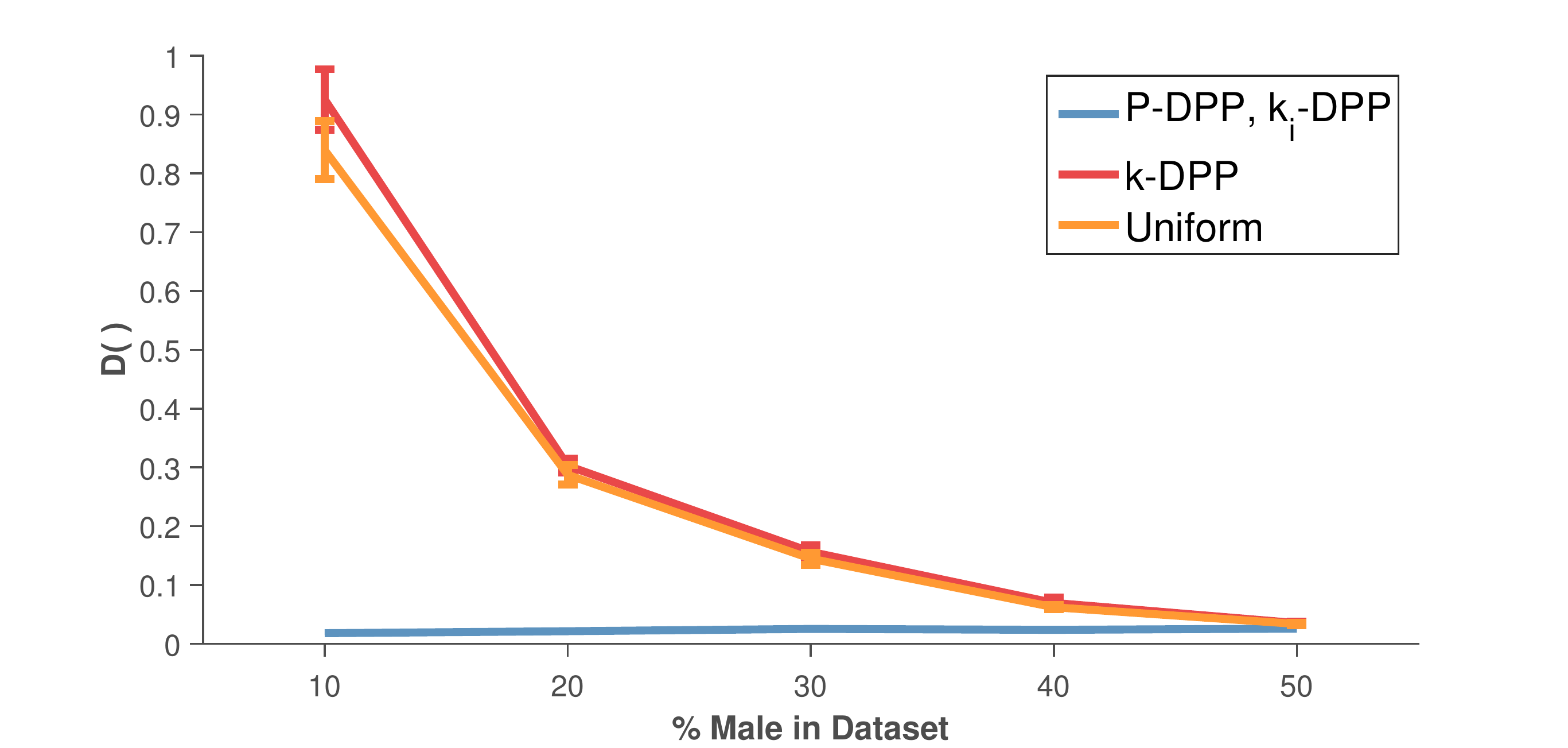}
\includegraphics[height=4.2cm]{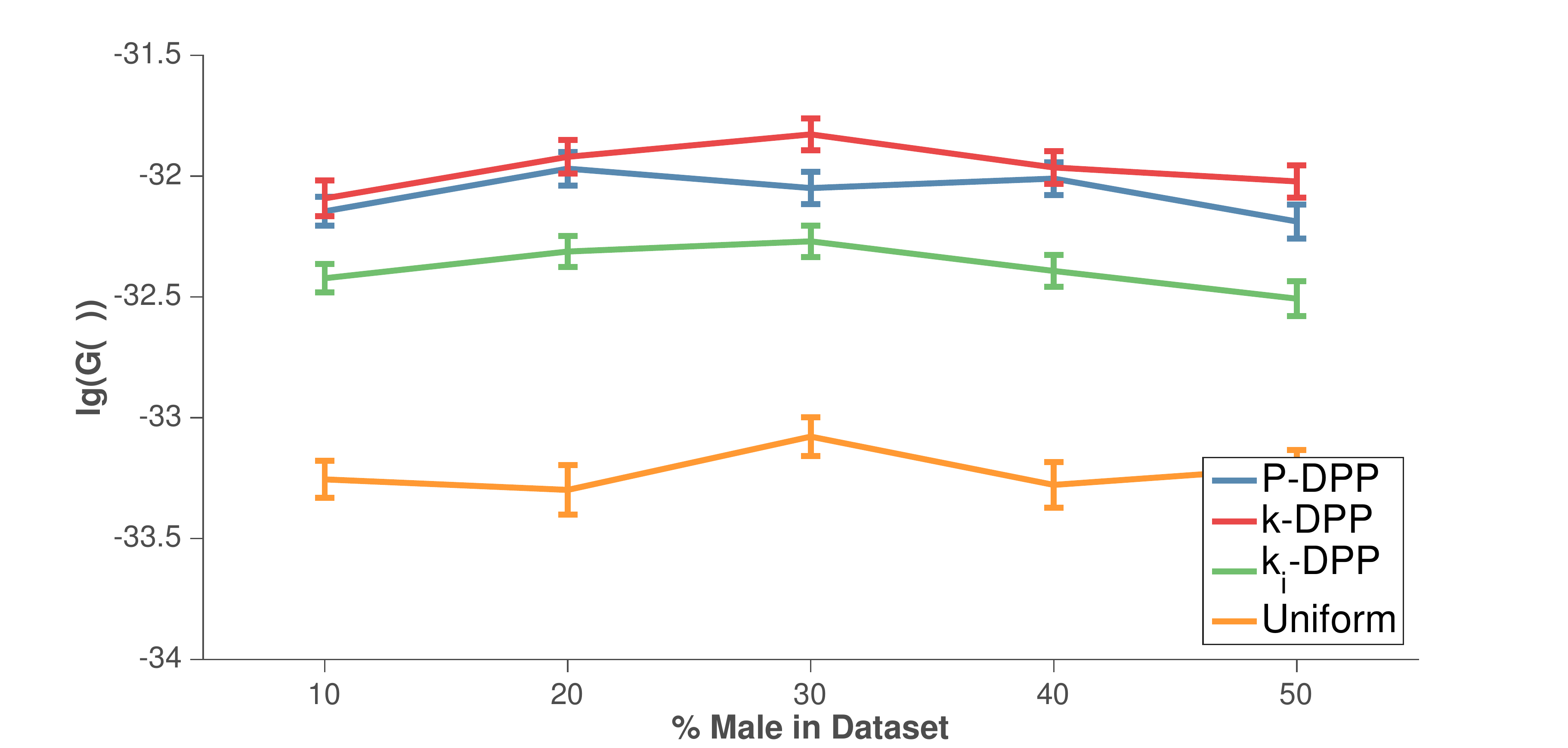}
\caption{\footnotesize {The mean relative unfairness measure $D(\cdot)=D^{\mathrm{un}}(\cdot)$ with respect to the uniform distribution over $4$ classes, and the logarithm of the geometric diversity $\lg(G(\cdot))$ are reported in the left and right figures respectively for $n = 200$ repetitions. Error bars represent the standard error of the mean.}}
\label{fig:images}
\end{figure*}


\subsection{$\beta$-balanced property for random data}\label{sec:betaConcentration}
For a given matrix $V \in \R^{m \times n}$, suppose we choose the partitions randomly. For each element $x \in X$, we put $x$ in $X_i$ with probability $1/p$. 
Using the Matrix Chernoff bounds \cite{tropp2012user}, we prove the following theorem.

\begin{theorem}\label{thm:beta-bound}
Assume that all the rows $v_j$ (for $j\in X=\{1,2,\ldots,m\}$) of $V\in \R^{m \times n}$ satisfy $v_j^\top (V^\top V)^{-1} v_j \leq \frac{\delta^2}{8p\log (np)}$, where $\delta \in (0,1)$ is a constant. 
If $X$ is randomly partitioned into $X=X_1\cup X_2 \cup \ldots \cup X_p$
then with probability at least $\frac{1}{e}$, the partition $X_1, \dots, X_p$ is $\beta$-balanced with respect to $V$, for 
$\beta = \sqrt{(1+\delta)p}$.
\end{theorem}

\noindent
To prove Theorem~\ref{thm:beta-bound} we will use the following matrix concentration inequality.
\begin{lemma}[Matrix Chernoff bound, see e.g. \cite{tropp2012user}] \label{thm:matrix-chernoff}
Given independent, random, Hermitian matrices $M_1, \dots, M_m$ that satisfy
\[M_i \succeq 0 \text{ and } \lambda_{\max}(M_i) \leq R ~~\text{ for all }i\]
it holds
\[\mathbb{P}\big[ \lambda_{\min}\big( \sum_{i=1}^m M_i \big) \leq (1 - \delta)\mu_{\min} \big] \leq n \cdot e^{-\delta^2\mu_{\min}/2R}\]
where $0 \leq \delta \leq 1$, $\mu_{\min} = \lambda_{\min}(\sum_{i=1}^m \mathbb{E}[M_i])$.
\end{lemma}
 
\begin{proof}[Proof of Theorem~\ref{thm:beta-bound}]
To use the Matrix Chernoff bound, we design our random experiment in the following way. We are given vectors $v_1, \dots, v_m \in \mathbb{R}^n$ which are rows of matrix $V \in \R^{m \times n}$. Note that the singular values $\sigma_1 \geq \cdots \geq \sigma_n$ are the eigenvalues of $M := V^\top V = \sum_{i=1}^m v_iv_i^\top$.  We will form partitions by putting each vector in $X_i$ with $1/p$ probability.

Consider the formation of one such partition $X_i$. Let $Y_j$ be the random variable taking value $v_jv_j^\top$ with probability $1/p$ and 0 with probability $(1-1/p)$. $X_i$ will be all those elements for which we do not sample 0. Then for this instance we have that
\[M_i := V_{X_i}^\top V_{X_i}  = \sum_{j=1}^m Y_j.\]
Let $u_j := (pV^\top V)^{-\frac{1}{2}} v_j$, $Z_j = u_j u_j^\top$ and $\widetilde{M}_i:=\sum_{j=1}^m Z_j$. Then it can be seen that 
 $$\mathbb{E}\insquare{\widetilde{M}_i} = I.$$ Let $\varepsilon = \delta/2.$
Note that 
\begin{align*}
(1-\varepsilon)\cdot I \preceq \widetilde{M}_i  \Leftrightarrow (1 - \varepsilon)\cdot M \preceq pM_i.
\end{align*}
We know that if $A \preceq B$, then for all $j$, $\lambda_j(A) \leq \lambda_j(B)$ -- see e.g. \cite{bhatia2013matrix}.
Therefore if we show that $(1-\varepsilon)\cdot I \preceq \widetilde{M}_i$, then for all $j \in \{1, \dots, n\}$, 
\[\lambda_j(M_i) \geq \frac{1-\varepsilon}{p}\lambda_j(M).\]
This implies that $V_{X_i}$ will satisfy the $\beta$-balanced condition for $\beta = \sqrt{\frac{p}{1-\varepsilon}}$.
To show that $\widetilde{M}_i \succeq (1-\varepsilon)\cdot I$ holds (with decent probability), it is enough to show that $\lambda_{\min}(\widetilde{M}_i) \geq (1- \varepsilon)$.  
We will show it  using Matrix concentration inequalities. But first we need to bound $\lambda_{\max}(Z_j)$.
\[\lambda_{\max}(Z_j) \leq \norm{u_j}^2 = pv_j^\top (V^\top V)^{-1} v_j \leq \frac{\varepsilon^2}{2\log (np)}.\]
Using Lemma~\ref{thm:matrix-chernoff}, we get
\begin{align*}
\mathbb{P}\insquare{\lambda_{\min}\inparen{\widetilde{M}_i}\leq (1- \varepsilon)} &\leq n \cdot e^{-\varepsilon^2/2R} \\
&=n\cdot e^{-\log (np)} = \frac{1}{p}.
\end{align*}
From the above two inequalities, we have that 
\begin{align*}
\mathbb{P}\insquare{\widetilde{M}_i \succeq (1-\varepsilon)\cdot I} &\geq 1 - \mathbb{P}\insquare{\lambda_{\min}\inparen{\widetilde{M}_i} \leq (1- \varepsilon)}\\
& \geq 1 - \frac{1}{p}.
\end{align*}
Hence the probability that all the partitions satisfy this $\beta$-balanced condition, for $\beta = \sqrt{\frac{p}{1-\varepsilon}}$, is atleast
\[\inparen{ 1- \frac{1}{p}}^p = \frac{1}{e}.\]
Since $\varepsilon = \delta/2$ and $ 0 \leq \delta \leq 1$, it can be seen that 
\[\frac{1}{1-\varepsilon} \leq 1 + 2\varepsilon = 1 + \delta.\]
Therefore the partition is $\beta$-balanced, for $\beta = \sqrt{(1+\delta)p}$, with probability $\geq 1/e$.
\end{proof}

\noindent
The quantity $v_j^\top (V^\top V)^{-1} v_j$ is also called the \textit{statistical leverage score} of $v_j$ with respect to $V^\top V$. For two partitions, the theorem states that if the leverage score of all rows is $O(\frac{1}{\log n})$, then the partitions are $\beta$-balanced for $\beta \approx \sqrt{2}$.

\section{Price of Fairness}
\label{sec:priceOfFairness}
In this section we present conditions under which the $k$-DPP and $P$-DPP distributions are close to each other. 
Note that the support of a $P$-DPP is a subset of the support of the corresponding $k$-DPP. 
Thus,  a natural definition of the \textit{price of fairness} is the KL-divergence between them.
\begin{definition}[Price of Fairness]
Given a matrix $V \in \R^{m \times n}$,  partitions $X_1, \dots, X_p$  and integers $k_1, \dots, k_p$, let $k = k_1 + \cdots + k_p$. Suppose $q$ is the distribution defined by $k$-DPP over subsets of size $k$ and $q^\star$ is the distribution defined by $P$-DPP over subsets with $k_i$ elements from each $X_i$.
Then, the price of fairness is $D_{KL}(q^\star||q)$.
\end{definition}
\noindent
We define the following property for the input data and analyze its price of fairness. 

\begin{definition}[$\delta$-drop] \label{defn:delta-drop}
For $0\leq \delta \leq 1$, the partition $X_1, \dots, X_p$ is called a $\delta$-drop partition with respect to $V$ and $k_1, \dots, k_p$ if for all $i \in \{1, \dots, p\}$,
$\sigma_{i,k_i+1} \leq \delta\sigma_{i,k_i}.$ Here $\sigma_{i,j}$ is the $j$-th largest singular value of $V_{X_i}$.
\end{definition}
\noindent
Roughly, this says that, if $\delta$ is small, then each of  the matrices $V_{X_i}$ is effectively a rank-$k_i$ matrix.
Such a notion of low effective rank appears frequently in the machine learning literature \cite{roy2007effective,drineas1999clustering}.
We prove the following theorem  that asserts that if the $\delta$-drop condition is satisfied, then we can be sure that most of the probability mass is concentrated on subsets which satisfy  partition constraints. In such a case, sampling a $k$ sized subset using any $k$-DPP algorithm will output a subset which satisfies partition constraints with high probability. 

\begin{theorem} \label{thm:price-of-fairness}
Let $\eps \in (0,1)$ and suppose that the partition $X_1, \dots, X_p$ is $\delta$-drop w.r.t. $V$ and $k_1, \dots, k_p$, with $\delta \leq \frac{\varepsilon}{nN_0}$ and $N_0 := \binom{k+p-1}{p-1}$.
If $n \geq \sqrt{2}k \cdot \big( \frac{\gamma}{\sigma_n} \big)^2$ (with $\gamma := \max\{\sigma_{i,1}\}_i$, where $\sigma_{i,1}$ is the largest singular value of $V_{X_i}$ and $\sigma_n$ is the smallest non-zero singular value of $V$)
then the price of ensuring fairness is $D_{KL}(q^\star||q) \leq \log \frac{1}{(1-\varepsilon)}.$
\end{theorem}

\noindent
We will use the following lemma in the proof.
\begin{lemma} \label{lem:fainess_KL_fraction}
For every  $\eps \in (0,1)$, if 
$$\sum_{S \in \mathcal{C \setminus B}} \det(V_SV_S^\top) \leq \eps \sum_{S \in \mathcal{C}} \det(V_SV_S^\top)$$ then
$$ D_{KL}(q^\star||q) \leq \log \frac{1}{(1-\varepsilon)}.$$
\end{lemma}
\begin{proof}
From the assumption it follows
\begin{align*}
(1 - \varepsilon) \sum_{S \in \mathcal{C}} \det(V_SV_S^\top) \leq \sum_{S \in \mathcal{B}} \det(V_SV_S^\top).
\end{align*}
Hence, for all $S \in \mathcal{C}$,
$$
\frac{\det(V_SV_S^\top)}{(1-\varepsilon)  \sum_{S \in \mathcal{C}} \det(V_SV_S^\top)} \geq \frac{\det(V_SV_S^\top)}{\sum_{S \in \mathcal{B} }\det(V_SV_S^\top)}, $$
which translates to
$$\frac{q^*(S)}{q(S)} \leq \frac{1}{(1-\varepsilon)}.$$
Finally, we obtain
$$D_{KL}(q^*||q) =  \sum_{S \in \mathcal{B}} q^*(S) \log \frac{q^*(S)}{q(S)} \leq \log  \frac{1}{(1-\varepsilon)}.$$
\end{proof}

\begin{proof}[Proof of Theorem~\ref{thm:price-of-fairness}]
We start by decomposing the terms in $\sum_{S \in \mathcal{C \setminus B}} \det(V_SV_S^\top) $ and analyzing each term individually using Lemma~\ref{lem:fainess_KL_fraction}. Given a set $S \subseteq X$, let $S_i := S \cap X_i$. Then $S = \bigcup_{i=1}^p S_i$. Using this, the family $\mathcal{C \setminus B}$ can be decomposed as
\begin{align*}
\mathcal{C \setminus B} &= \{S \subseteq X\mid \exists j ~~ |S \cap X_j| \neq k_j\} \\
&= \inbraces{\bigcup_{i=1}^p S_i \mid \forall j~~ S_j \subseteq X_j \text{ and } \exists j ~~ |S_j| \neq k_j}.
\end{align*}

\noindent
Let $S_{(j_1, \dots, j_p)}$ denote the following family of subsets
\[S_{(j_1, \dots, j_p)} := \{S \subseteq X \mid |S \cap X_i| = j_i \}\]
and, for brevity, let $\mathcal{J}$ denote the following set integer tuples (all but $(k_1, k_2, \ldots, k_p)$)
\[\mathcal{J} := \N^{p}_{\geq 0} \setminus \{(k_1, k_2, \ldots, k_p)\}.\]

\noindent
Given this notation, we can write the following sum as
\begin{align*}
\sum_{S \in \mathcal{C \setminus B}} \det(V_SV_S^\top) = \sum_{(j_1, \dots, j_p) \in \mathcal{J}} \sum_{S \in S_{(j_1, \dots, j_p)}} \det(V_SV_S^\top).
\end{align*}

\noindent
We analyze each term of the above summation individually. We start by noting that 
\[\det(V_{S}V_{S}^\top)  \leq \prod_{i=1}^p \det(V_{S_i}V_{S_i}^\top), \]
where for all $i$, $S_i = S \cap X_i$, this is a simple consequence of the fact that $VV^\top$ is positive semidefinite.
Therefore,
\begin{align*}
\sum_{S \in S_{(j_1, \dots, j_p)}} \det(V_SV_S^\top) \leq \prod_{i=1}^p \sum_{S_i \subseteq X_i, |S_i| = j_i}\det(V_{S_i}V_{S_i}^\top). 
\end{align*}

\noindent
Whenever a set $S$ of cardinality $k$ does not belong to $\cB$, for at least one $i$, we have that $|S_i|=|S \cap X_i|>k_i$.
Let us now analyze how does a sum of the form $\sum_{T\subseteq X_i, |T|=j} \det(V_T V_T^\top)$ behave depending on whether $j\leq k_i$ or $j>k_i$.

\noindent
\textbf{Case 1. $j\leq k_i$ :}
\begin{align*}
\sum_{\substack{T\subseteq X_i, |T| = j}} \det(V_{T}V_{T}^\top)  &= \sum_{1\leq l_1 < \cdots < l_j \leq n} \prod_{j'=1}^{j} \sigma_{i,l_{j'}}^2 \\
&\leq \sum_{l=0}^{j} \binom{k_i}{l} \gamma^{2l} \binom{n-k_i}{j-l} (\gamma\delta)^{2(j-l)}\\
&= \gamma^{2j}\sum_{l=0}^{j} \binom{k_i}{l} \binom{n-k_i}{j-l} \delta^{2(j-l)}\\
&\leq \gamma^{2j}\sum_{l=0}^{j} \binom{k_i}{l} (n-k_i)^{j-l} \delta^{2(j-l)}.
\end{align*}
Since $\delta < \frac{\varepsilon}{nN_0}$,
\[\sum_{\substack{T\subseteq X_i, |T| = j}} \det(V_{S}V_{S}^\top)  \leq \gamma^{2j}2^{k_i}.\]

\noindent
\textbf{Case 2. $j>k_i$ :} 
\begin{align*}
\sum_{\substack{T\subseteq X_i, |T| = j}} \det(V_{T}V_{T}^\top)  &= \sum_{1\leq l_1 < \cdots < l_j \leq n} \prod_{j'=1}^{j} \sigma_{i,l_{j'}}^2 \\
&\leq \sum_{l=0}^{k_i} \binom{k_i}{l} \gamma^{2l} \binom{n-k_i}{j-l} (\gamma\delta)^{2(j-l)}\\
&= \gamma^{2j}\sum_{l=0}^{k_i} \binom{k_i}{l} \binom{n-k_i}{j-l} \delta^{2(j-l)}\\
&= \gamma^{2j}\sum_{l=0}^{k_i} \binom{k_i}{l} (n-k_i)^{j-l} \delta^{2(j-l)}.\\
\end{align*}
Since $\delta < \frac{\varepsilon}{nN_0}$,
\begin{align*}
\sum_{\substack{T\subseteq  X_i, |T| = j}} \det(V_{T}V_{T}^\top)  &\leq \big( \frac{\varepsilon}{N_0} \big)^{j-k_i} \gamma^{2j}\sum_{l=0}^{k_i} \binom{k_i}{l} \frac{1}{n^{j-l}}.
\end{align*}
Since $j  > k_i$, we have
\[\frac{1}{n^{j-l}} \leq \frac{1}{{k_i}^{j-l}} \leq \frac{1}{{k_i}^{k_i-l} \cdot k_i}\]
and 
\[\binom{k_i}{l} \frac{1}{n^{j-l}} \leq {k_i}^{k_i-l}\frac{1}{{k_i}^{j-l}\cdot k_i} \leq \frac{1}{k_i}.\]
Therefore,
\[\sum_{\substack{T\subseteq  X_i, |T| = j}} \det(V_{T}V_{T}^\top)  \leq \inparen{ \frac{\varepsilon}{N_0} }^{j-k_i} \gamma^{2j} \leq \frac{\varepsilon}{N_0}\gamma^{2j}.\]

\noindent
Using the above inequalities, we obtain that for every $(j_1, \ldots, j_p)\in \mathcal{J}$
\begin{align*}
\sum_{S \in S_{(j_1, \dots, j_p)}}\det(V_S V_S^\top) \leq \frac{\varepsilon}{N_0}\gamma^{2k}2^k.
\end{align*}
Note that the size of the set of tuples $\mathcal{J}$ is bounded from above by $|\mathcal{J}| \leq \binom{k+p-1}{p-1}=N_0$. Therefore,
\begin{align*}
\sum_{S \in \mathcal{C \setminus B}} \det(V_SV_S^\top) 
&= \sum_{(j_1, \dots, j_p) \in \mathcal{J}} \sum_{S \in S_{(j_1, \dots, j_p)}} \det(V_SV_S^\top) \\
&\leq N_0 \cdot \frac{\varepsilon}{N_0}\gamma^{2k}2^k = \varepsilon \gamma^{2k}2^k.
\end{align*}

\noindent
It remains to find a lower bound for $\sum_{S \in \mathcal{C}} \det(V_SV_S^\top)$. Using Lemma~\ref{lemma:singularValues}, we obtain
\begin{align*}
\sum_{S \in \mathcal{C}} \det(V_SV_S^\top) &= \sum_{1\leq i_1 < \cdots < i_{k} \leq n} \prod_{j=1}^k \sigma_{i_{k}}^2 \geq \binom{n}{k} \cdot \sigma_n^{2k}.
\end{align*}
By using the inequality $\binom{n}{k} \geq \frac{n^k}{k^k}$ we finally arrive at
\[\sum_{S \in \mathcal{C}} \det(V_SV_S^\top) \geq \inparen{\frac{n}{k} \sigma_n^2 }^{k}. \]

\noindent
Therefore,
\begin{align*}
\frac{\sum_{S \in \mathcal{C \setminus B}} \det(V_SV_S^\top)}{\sum_{S \in \mathcal{C}} \det(V_SV_S^\top)} &\leq \frac{\varepsilon \gamma^{2k}2^k}{ \inparen{ \frac{n}{k} \sigma_n^2 }^{k} } \leq \varepsilon  \cdot \inparen{\frac{\sqrt{2}k\gamma^2}{n\sigma_n^2}}^{k}.
\end{align*}
Using the assumption that $n \geq \sqrt{2}k \cdot \big( \frac{\gamma}{\sigma_n} \big)^2$ we obtain
\[\sum_{S \in \mathcal{C \setminus B}} \det(V_SV_S^\top) \leq \varepsilon \sum_{S \in \mathcal{C}} \det(V_SV_S^\top).\]
and an application of Lemma~\ref{lem:fainess_KL_fraction} finishes the proof.
\end{proof}

\section{Empirical Results} \label{sec:ExpResults}
\renewcommand{\thesubfigure}{(\roman{subfigure})}

\subsection{Algorithms and Baselines}

In each experiment, we compare several different probability distributions from which to select $k$ samples from a dataset: 
As benchmarks we consider the (unconstrained) distributions, $k$-DPP (see Def~\ref{def:kdpp}), and UNIF, which selects a uniformly random subset of size $k$ from the dataset $X$. We compare this against different methods which select from a fair family of allowed subsets, $P$-DPP (see Def~\ref{def:pdpp}), and $k_i$-DPP (see Def~\ref{def:kidpp} below). 

\begin{definition}{\bf ($k_i$-DPP)}
\label{def:kidpp}
Given a dataset $X$, the corresponding feature vectors $V \in \mathbb{R}^{m \times n}$,  a partition $X = X_{1} \cup \cdots \cup X_{p}$ into $p$ parts, and numbers $k_1,\ldots,k_p$,  $k_i$-DPP defines a distribution over $k_1+\cdots+k_p$-sized subsets $S \subseteq X$ that is a product distribution:  for each $i,$ we obtain a  sample $S_{i} \subseteq X_{i}$ of size $k_i$ independently with probability proportional to $\mathbb{P}[{S_i}] \propto \deter{V_{S_i}V_{S_i}^\top}$, and combine these samples to output $S = S_{1} \cup \dotsb \cup S_{p}$.
\end{definition}
Algorithms for $k_i$-DPPs are simply obtained by  {\em independently} using a $k$-DPP sampler with $k = k_i$ on each part $X_i$.
For sampling from all the above listed distribution we use the  Sample and Project algorithm as described in Section~\ref{sec:sample_project}.

\subsubsection{Metrics}
In each experiment, we report the geometric diversity $G( \cdot )$ (see Def~\ref{def:diversity}) and the fairness as measured by the KL-divergence from the desired frequency over parts.
Formally, given a probability distribution $q$ over the $p$ parts of the dataset, we define the relative unfairness measure of a set $S\subseteq X$ as 
$D^q(S):=D_{KL}(q || s),$
where $s=(s_1, \ldots, s_p)$ denotes the vector of frequencies, i.e., $s_i=\frac{|X_i \cap S|}{|S|}$ for $i=1,2, \ldots, p$.
In particular, typically we want to have $D^q(\cdot)$ as small as possible -- ideally equal to $0$.
When $q_i=1/p$ for all $i$, we refer to $D^q$ as $D^{\rm un}$. 
When $q_i=|X_i|/m$, we refer to $D^q$ as $D^{\rm prop}$. 
%


\subsection{Experiment on Image Dataset}
\label{sec:exp1}

\subsubsection{Curated Dataset}
We gathered a collection of images curated using Google image search as follows:
Four search terms were used: (a) ``Scientist Male'', (b) ``Scientist Female'', (c) ``Painter Male'', and (d) ``Painter Female''.\footnote{The images are available at \url{goo.gl/hNukfP}.}

Following \cite{Kulesza2011}, each image was processed with the \texttt{vlfeat} toolbox to obtain sets of 128-dimensional SIFT descriptors \cite{Lowe99,vlfeat}. All such descriptors are collected in a single set and subsampled to roughly $10\%$ of its total size. The resulting set of  $\approx10^4$ descriptors was clustered using the $k$-means algorithm where $k=128$ is the number of means. The feature vector for an image is  the normalized histogram of the nearest clusters to the descriptors in the image.


\subsubsection{Experiment on Biased Datasets}
Our goal is to understand how the bias in the underlying dataset can affect the performance of the different sampling distributions with respect to fairness and geometric diversity. We include all female (b and d) images, but vary how many of the male images (a and c) appear in the dataset in order to create biased sets that have between $10\%$ to $50\%$ male images. The male images are selected uniformly at random from the set of all male scientists and male artists for each repetition in the experiment.
We sample 40 images from each biased dataset; roughly the number that fits on the first page of an image search result. We conduct 200 repetitions.
We place fairness constraints so that $P$-DPP and $k_i$-DPP select exactly $50\%$ of their samples from the male (a and c) images and female (b and d) images, {regardless of the bias in the underlying dataset}. Note that we \emph{do not} enforce constraints across scientist (a and b) images and artist (c and d) images, but measure the unfariness $D^{\rm un}(\cdot)$ %
with respect to all four attributes.

\subsubsection{Results}
With respect to  $D^{\rm un}(\cdot)$, $P$-DPP significantly outperforms $k$-DPP, and UNIF (paired one-sided $t$-tests, $p < 0.05$), see Figure~\ref{fig:images}. As expected, the bias in the underlying dataset can dramatically affect the fairness of UNIF and $k$-DPP as neither approach is designed to correct for such biases. However, $P$-DPP and $k_i$-DPP both enforce fairness constraints; note that this is despite the fact that the sampling was only equal with respect to gender and not profession. The latter does not appear to affect the outcome here.

With respect to the diversity $G(\cdot)$, $P$-DPP has significantly higher $G(\cdot)$ than UNIF and $k_i$-DPP (paired one-sided $t$-tests, $p < 0.05$). Moreover, $P$-DPP performs comparatively to $k$-DPP; the mean diversity of $k$-DPP is higher, but not significantly so. 
Thus, we observe that, when the underlying data is biased, there is a tradeoff between $D^{\rm un}(\cdot)$ (for which $P$-DPP performs best) and $G(\cdot)$ (for which $k$-DPP performs best); however the differences in geometric diversity are negligible while differences in unfairness can be very large.

\subsection{Experiment on Real-World Dataset}

\begin{tiny}
\begin{table*}[] \centering
\fontsize{10}{12}\selectfont
\scalebox{0.90}{
\begin{tabular}{@{}rl|rr|rr|rr|rr|rr|rr|@{}}
   \multicolumn{2}{c|}{}   & \multicolumn{6}{c|}{\textbf{Gender}} & \multicolumn{6}{c|}{\textbf{Race}}  \\ 
   \multicolumn{2}{c|}{}  & \multicolumn{2}{c}{$D^{\mathrm{un}}(\cdot)$} & \multicolumn{2}{c}{$D^{\mathrm{prop}}(\cdot)$} & \multicolumn{2}{c|}{$\log G(\cdot)$} & \multicolumn{2}{c}{$D^{\mathrm{un}}(\cdot)$} &
   \multicolumn{2}{c}{$D^{\mathrm{prop}}(\cdot)$} &
     \multicolumn{2}{c|}{$\log G(\cdot)$}   \\ 
     \multicolumn{2}{c|}{\textbf{Sampling Met.}} & 
   \multicolumn{1}{c}{\textbf{mean}} & \multicolumn{1}{c|}{\textbf{std}} &
   \multicolumn{1}{c}{\textbf{mean}} & \multicolumn{1}{c|}{\textbf{std}} & \multicolumn{1}{c}{\textbf{mean}} & \multicolumn{1}{c|}{\textbf{std}} & \multicolumn{1}{c}{\textbf{mean}} & \multicolumn{1}{c|}{\textbf{std}} &
   \multicolumn{1}{c}{\textbf{mean}} & \multicolumn{1}{c|}{\textbf{std}} & \multicolumn{1}{c}{\textbf{mean}} & \multicolumn{1}{c|}{\textbf{std}} \\
   \midrule
\multirow{2}{*}{\textbf{Uncon.}} 
&UNIF			&0.075	&0.019 &0.001	&0.002 	&-67	&41	&0.357	&0.050&0.001	&0.001  	&-67	&41\\
&$k$-DPP		&0.027	&0.009&0.011	&0.005 	&489	&11	&0.268	&0.038	&0.005	&0.004&487	&12\\
        \hdashline 
\multirow{3}{*}{\textbf{Equal}}   
&$k_i$-UNIF	&0	&0	&0.069	&0 	&-31	&35	&0	&0		&0.282	&0&16	&32\\
&$k_i$-DPP	&0	&0		&0.069	&0 &410	&16	&0	&0		&0.282	&0&366	&16\\
&$P$-DPP		&0	&0	&0.069	&0 	&490	&11	&0	&0	&0.282	&0	&476	&12\\
                \hdashline
\multirow{3}{*}{\textbf{Prop.}}  
&$k_i$-UNIF	&0.074	&0	&0	&0 	&-64	&29	&0.358	&0	&0	&0	&-65	&35\\
&$k_i$-DPP	&0.074	&0		&0	&0 &409	&17	&0.358	&0		&0	&0&426	&15\\
&$P$-DPP		&0.074	&0	&0	&0 	&482	&13	&0.358	&0	&0	&0	&488	&12\\
\bottomrule
\end{tabular}
}
\caption{We report the unfairness ($D^{\mathrm{un}}(\cdot)$ with respect to the uniform distribution over parts, and $D^{\mathrm{prop}}(\cdot)$ with respect to the ``proportional'' distribution, i.e. as in the whole dataset) and diversity ($\log G(\cdot)$) for the different sampling methods on the Adult dataset when (a) the sensitive attribute is Gender or (b) the sensitive attribute is Race. Sets of size 400 were selected, and 100 samples were taken for each. For the samplers that match fairness constraints, we consider both selecting subsets with equal representation and selecting subsets with proportional representation. We note that $P$-DPP has the highest diversity out of all constrained sampling methods regardless of the method of representation. Moreover, the diversity of $P$-DPP matches that of the unconstrained $k$-DPP for Gender under proportional representation and for Race under equal representation.}
\label{tab:adult}
\end{table*}
\end{tiny}

\subsubsection{The Adult Dataset}  
The Adult income dataset~\cite{datasetUCI} consists of roughly 45000 records of subjects each with 14 features such as age, race, education and a binary label indicating whether a subject's incomes is above or below 50K USD.\footnote{Data downloaded from \url{https://archive.ics.uci.edu/ml/datasets/adult}.} 
This dataset has been widely studied in the context of fairness (see,  \cite{YangS17,ZafarVGG17,ZemelWSPD13,Zadrozny04}).

In preprocessing the data we filter out incomplete entries, and from the remaining ones we pick a random subset of 5000 records for our experiments. 
We vectorize the data as follows: 
Categorical fields (with a small number of possible values) we turn into sets of binary fields.
As the dimension $n$ of such feature vectors is quite small -- $50$ -- the DPP framework allows sampling sets of cardinality at most $k\leq 50$. For this reason we enrich the feature vectors in a standard way -- by adding pairwise products of all existing features as separate ones -- this, after removing redundant columns, yields feature vectors of dimension $992$.

\subsubsection{Experiment on Equal and Proportional Representation}
We conduct our experiment across either gender or race as the sensitive attribute. For the former, we use the gender categories provided in the dataset; all entries were labeled either male (68.3\%) or female (31.7\%). 
For the latter, we use the race categories provided in the dataset; we consider the partition Caucasian (85.7\%) and non-Caucasian (14.3\%).

In addition to the algorithms mentioned above, we report the performance of an additional benchmark $k_i$-UNIF, which selects a uniformly random subset of size $k_i$ from $X_i$.

In our subsampling, we consider both equal representation, where each attribute makes up of 50\% of the selected points, and proportional representation, where each attribute is represented with the same ratio as in the original population.

\subsubsection{Results}
We observe that $P$-DPP has the highest diversity out of all constrained sampling methods regardless of the proportion of representation or sensitive attribute; see Table \ref{tab:adult}. 
Surprisingly, the diversity of $P$-DPP matches that of the unconstrained $k$-DPP for Gender under proportional representation and for Race under equal representation. In the other two settings -- Gender under equal representation and Race under proportional representation -- the $P$-DPP score is lower than that of $k$-DPP, but minimally so, and outperforms $k_i$-DPP by several standard deviations. 

We note that  $k_i$-UNIF, although it has very poor geometric diversity as a whole, performs better under equal representation than it does under proportional representation. This fact suggests that there could be value in selecting sensitive attributes equally beyond the consideration of fairness. 

The fact that $P$-DPP performs so well, especially when significantly changing the distribution of sensitive attributes (e.g., for race, from 14.3\% non-Caucasian to 50\% non-Caucasian), is quite surprising. Overall, it appears that one can support very dramatic changes to the underlying distributions of attributes with minimal or even zero loss to geometric diversity by using our $P$-DPP algorithm.

\subsection{Experiment on Price of Fairness}
We look at the effect of the scaling of singular values, suggested by Theorem~\ref{thm:price-of-fairness}, on the sampled subsets of our Algorithm. In this experiment we take an instance of random vectors and use different sampling methods to sample a subset from the dataset, and report the $D^{\rm un}(\cdot)$ and $\log G(\cdot)$ value of the sampled subset. Following this, we scale the tail singular values of the partition matrices by $\delta = O(1/n)$ and again report the $D^{\rm un}(\cdot)$ and $\log G(\cdot)$ values.

We also present a heuristic approach, Scale-And-Sample, for constrained sampling which will use any $k$-DPP algorithm as a sub-routine. The algorithm is simple. For each $V_{X_i}$, scale the  smallest $(n-k_i)$ singular values by $1/n$. Then sample a $\sum_{i=1}^p k_i$ sized subset using any $k$-DPP algorithm.

\begin{table*}[t] \centering
\fontsize{10}{12}\selectfont
\scalebox{0.90}{
\begin{tabular}{@{}rl|rr|rr|rr|rr|@{}}
   \multicolumn{2}{c|}{}   & \multicolumn{4}{c|}{\textbf{Before Scaling}} & \multicolumn{4}{c|}{\textbf{After scaling}}  \\ 
   \multicolumn{2}{c|}{}  & \multicolumn{2}{c}{$D^{\rm un}(\cdot)$} & \multicolumn{2}{c|}{$\log G(\cdot)$} & \multicolumn{2}{c}{$D^{\rm un}(\cdot)$} &  \multicolumn{2}{c|}{$\log G(\cdot)$}   \\ 
     \multicolumn{2}{c|}{\textbf{Sampling Method}} & 
   \multicolumn{1}{c}{\textbf{mean}} & \multicolumn{1}{c|}{\textbf{std}} & \multicolumn{1}{c}{\textbf{mean}} & \multicolumn{1}{c|}{\textbf{std}} & \multicolumn{1}{c}{\textbf{mean}} & \multicolumn{1}{c|}{\textbf{std}} & \multicolumn{1}{c}{\textbf{mean}} & \multicolumn{1}{c|}{\textbf{std}} \\
   \midrule
\multirow{2}{*}{\textbf{Unconstrained}} 
&UNIF			&0.066	&0	&455.7	&1.4	&0.064	&0	&228.6	&215.8\\
&$k$-DPP	&0.063	&0	&457.3	&1.3	&$5.2 \times 10^{-6}$ &0	&397.4	&11.6\\
&Scale-And-Sample	&$5.2 \times 10^{-6}$  &0	&457.5	&1.1	&- &-	&- 	&-\\
        \hdashline
\multirow{3}{*}{\textbf{Constrained}}   
&$k_i$-UNIF	&0	&0		&455.7		&1.3		&0	&0		&226.5		&20.8\\
&$P$-DPP		&0	&0		&457.2		&1.1		&0	&0		&397.5		&9.2\\
\bottomrule
\end{tabular}
}
\caption{We report the unfairness ($D^{\mathrm{un}}(\cdot)$ with respect to the uniform distribution over parts) and diversity for the different sampling methods on a random dataset before and after scaling the singular values by a factor of $1/n$. In this experiment we have $m = 200$ vectors of dimension $n=150$ divided into two partitions (partition 1 has $\frac{m}{3}$ elements and partition 2 has $\frac{2m}{3}$ elements), and we want to sample 50 elements from each partition ($k = 100$).}
\label{tab:fairness}
\end{table*}

\subsubsection{Results}
The results are presented in Table~\ref{tab:fairness}.
It can be seen that after scaling the tail singular values of the partition matrices, the mean $D^{\rm un}(\cdot)$ value for $k$-DPP is very low, and resembles closely the constrained sampling case.
We also note that the Scale-And-Sample approach to constrained sampling suggested earlier performs very well. The mean relative unfairness measure $D^{\rm un}(\cdot )$ is almost zero. Furthermore, the value of the geometric diversity parameter $\log G(\cdot)$ is also similar to unscaled $P$-DPP.

\section{Proofs}\label{sec:proofs}

\subsection{Proof of Lemma~\ref{lem:opt-kl-dist}} \label{proof-opt-kl}
\begin{proof}
We need to show that $q^\star$, as defined below, is the optimal (closest to $\tilde{q}$ in $KL$-distance) distribution over $\mathcal{C}$
\begin{equation*}
q^\star(S) =     
\begin{cases}
  \alpha \cdot \tilde{q}(S) & \text{for } S \in \mathcal{C}\\    
  0 &  \text{ otherwise}   
\end{cases}
\end{equation*}
where $\alpha = 1/ \sum_{S \in \mathcal{C}} \tilde{q}(S)$. Note first that
$D_{KL}(q^\star|| \tilde{q}) = \log \alpha$. Consider any distribution $q$ over $\mathcal{C}$, it remains to show that $D_{KL}(q || \tilde{q}) \geq \log \alpha$. We have
\begin{align*}
D_{KL}(q || \tilde{q}) &=\sum_{S\in \cC} q_S \log \frac{q_S}{\tilde{q}_S}\\
&= \sum_{S\in \cC} q_S \log \frac{q_S}{\alpha \tilde{q}_S}+\log \alpha\\
&= D_{KL}(q || q^\star) + \log \alpha\\
& \geq \log \alpha,
\end{align*}
since $D_{KL}(q || q^\star) \geq 0$. Therefore, the minimum possible value of $D_{KL}(q || \tilde{q})$ is $\log \alpha$, which is achieved for $q = q^\star$.

\end{proof}

\subsection{Proof of Lemma~\ref{lemma:det_product}} \label{proof:det_product}
\begin{proof}
We will prove this lemma by induction.
\noindent
For the base case where there is just one row in $W$, $\det(WW^\top)$ is equal to $\norm{w_1}^2$ which is equal to $\norm{\Pi_{H_1}w_1}^2$.

Let $W'$ be the matrix with $\{w_1, \dots, w_{k-1}\}$ as rows. Assume that the statement is true for $k-1$ rows, i.e., \[\det(W'W'^\top)=\prod_{i=1}^{k-1} \norm{\Pi_{H_i}w_i}^2.\] 
\noindent
Then for $W$ we have, 
\begin{align*}
    WW^\top &= \begin{bmatrix}
           w_k \\
           W' \\
         \end{bmatrix}  \begin{bmatrix}
           w_k^\top & W'^\top 
         \end{bmatrix}
         = \begin{bmatrix}
           \norm{w_k}^2 &  W'^\top w_k \\
           w_k^\top W' & W'W'^\top\\
         \end{bmatrix}.
  \end{align*}
The first row of this matrix is \begin{align*}
     \begin{bmatrix}
           w_k^\top w_k & w_k^\top w_{k-1} & \dots & w_k^\top w_1 
         \end{bmatrix}.
  \end{align*}
\noindent 
Note that elementary row product or addition transformations do not change the determinant. We will apply these transformation to make the entries of first row and first column go to zero.

Let $(i)$ denote the $i$-th row of the above matrix and $WW^\top_{(i,j)}$ denote the $(i,j)$ entry. Then the transformation
\[(1) - \frac{w_k^\top w_{k-1}}{w_{k-1}^\top w_{k-1}} (2) \]
will make the $WW^\top_{(1,2)}$ entry go to zero. For the rest of the elements,
\begin{align*}
WW^\top_{(1,i)} &= w_k^\top w_{k-i+1} - \frac{w_k^\top w_{k-1}}{w_{k-1}^\top w_{k-1}} w_{k-1}^\top w_{k-i+1} \\
&= w_{k-i+1}^\top \Pi_{w_{k-1}}(w_k).
\end{align*}
In particular,
\begin{align*}WW^\top_{(1,1)} &= w_k^\top w_k - \frac{w_k^\top w_{k-1}}{w_{k-1}^\top w_{k-1}} w_{k-1}^\top w_k.\\
& = w_k^\top \Pi_{w_{k-1}}(w_k).
\end{align*}
We continue this way and next apply the transformation 
\[(1) - \frac{ w_{k-2}^\top \Pi_{w_{k-1}}(w_k)}{w_{k-2}^\top w_{k-2}} (3). \]
This will make the $WW^\top_{(1,3)}$ entry go to zero and by the similar analysis as above we get $WW^\top_{(1,i)} =  w_{k-i+1}^\top \Pi_{H'_2}(w_k)$, 
where $H'_i$ is the subspace spanned by the vectors $\{w_{k-1}, \dots, w_{k-i}\}$. After applying $k-1$ row transformations of the form
\[(1) - \frac{ w_{k-j+1}^\top \Pi_{H'_{j-1}}(w_k)}{w_{k-j+1}^\top w_{k-j+1}} (j)\]
we get that the entries $WW^\top_{(1,i)} = 0$, for $i \neq 1$ and 
\[WW^\top_{(1,1)} =  w_k^\top \Pi_{H'_k}(w_k) = \norm{\Pi_{H'_k}(w_k)}^2.\]
Note that $H'_k = H_k$ defined in the statement of the lemma. 

We can apply similar column operations to make all the entries of the first column, except $WW^\top_{(1,1)}$, go to zero. Since these elementary operations do not affect the determinant, we get
Therefore 
\begin{align*}
    \det(WW^\top) &= 
         \det \begin{bmatrix}
           \norm{w_k}^2 &  W'^\top w_k \\
           w_k^\top W' & WW'^\top
         \end{bmatrix}\\
         &= \det \begin{bmatrix}
           \norm{\Pi_{H_k}(w_k)}^2 & 0 \\
           0 & W'W'^\top\\
         \end{bmatrix}.
\end{align*}
Using the induction hypothesis we get,
\begin{align*}
\det(WW^\top) &= \norm{\Pi_{H_k}(w_k)}^2 \cdot \det(W'W'^\top) \\
&= \prod_{i=1}^{k} \norm{\Pi_{H_i}(w_i)}^2.
\end{align*}
\end{proof}

\subsection{Proof of Lemma~\ref{lemma:singularValues}} \label{proof:singularValues}
\begin{proof}
Consider two forms of the characteristic polynomial of the matrix $-VV^\top\in \R^{m \times m}$, i.e.,
\[\det(xI + VV^\top) = \prod_{i=1}^m (x + \sigma_i^2),\]
where $\sigma_1, \dots, \sigma_m$ are the singular values of $V$.

The coefficient of $x^{m-k}$ in $\prod_{i=1}^m (x + \sigma_i^2)$ is equal to $\sum_{1\leq i_1 < i_2<\ldots < i_{k}\leq m} \sigma_{i_1}^2 \sigma_{i_2}^2 \cdot  \cdots \cdot \sigma_{i_{k}}^2$.
Let $\mathcal{W}_k$ be the set of all principal $k$-minors of $VV^\top$. It is a well known fact in linear algebra that the coefficient of $x^{m-k}$ in $\det(xI + VV^\top)$ is equal to 
\[\sum_{W \in \mathcal{W}_k} \det(W) = \sum_{S : |S| = k} \det(V_SV_S^\top).\]
Therefore,
\[\sum_{i_1 < i_2< \cdots < i_{k}} \sigma_{i_1}^2 \sigma_{i_2}^2 \cdot  \cdots \cdot  \sigma_{i_{k}}^2 = \sum\limits_{S:|S| = k} \det(V_SV_S^\top)\]
\end{proof}

\subsection{Proof of Lemma~\ref{lemma:nondegeneracy}} \label{proof:nondegeneracy}
\begin{proof}
We first show that for every part $i$, the corresponding matrix $V_{X_i}$ has rank at least $k$. 
For this, first note that $V$ has at least $k$ non-zero singular values, i.e., $\sigma_k>0$. This follows from the fact that the number of non-zero singular values determines the rank of $V$. The rank of $V$ is certainly at least $k$, since otherwise the diversity of every subset of size $k$ would be zero.

From the $\beta$-balance condition it follows that the number of non-zero singular values of $V_{X_i}$ is the same as for $V$, and hence also the rank of $V_{X_i}$ is at least $k$, as claimed.

Note now that the set of vectors output by the algorithm has determinant zero if and only if for an iteration $j$ there exists a partition $X_i$ such that $|S \cap X_i| < k_i$ and $\norm{w_x} = 0$ for all $x \in X_i$, where $S = \{x_1, \dots ,x_{j-1}\}$.

This is equivalent to saying that all vectors in $V_{X_i}$ belong to the subspace spanned by the vectors in $S$. 
Since the size of $S$ is $j-1$, the dimension of the subspace spanned by the vectors in $V_S$ is at most $j-1$.
Since, by assumption for every $x\in X_i$ the projection of $v_x$ onto the subspace $\mathrm{span}\{v_y: y\in S\}$ is $0$, it implies that the dimension of subspace spanned by vectors in $V_{X_i}$ is less than $j \leq k$. This would contradict the claim proved at the very beginning -- that this dimension is at least $k$, hence the lemma follows.
\end{proof}

\section{Conclusion and Future Work}
\label{sec:discussion}

In this paper we initiated the study of fair and diverse DPP-based sampling for data summarization. 
We provide a novel and fast algorithm that can sample from a DPP that satisfy fairness constraints based on the desired proportion of samples with a given attribute. 
Our algorithm gives provably good guarantees when the data matrix satisfies a  natural $\beta$-balance property.
We prove that a large class of datasets satisfy the $\beta$-balance condition.
We define a notion of {\em price of fairness}, the KL-divergence between the fairness constrained distribution and the unconstrained distribution and theoretically show that, when the data satisfies reasonable properties, this price would be low.
We further show experimentally that adding fairness constraints results in minimal loss to diversity, even when the underlying dataset is very biased, or when the proportion of attributes is changed significantly.

Several challenging problems remain from a technical standpoint; naturally, a first question would be whether the theorems can be improved either by attaining better approximation guarantees, or by weakening the necessary conditions. Extending these results to arbitrary group structures (as opposed to partitions) would be very relevant, but appears to be significantly more challenging. 

From a practical point of view, it remains to be seen what effect de-biasing a sampler has on the end result of a machine learning algorithm (e.g., classification), both on its accuracy and on the bias down the line.

\bibliographystyle{plain}
\bibliography{references}

\end{document}